%% file: neurips_2026.tex
\documentclass{article}

 \usepackage[preprint]{neurips_2026}

\usepackage[utf8]{inputenc} % allow utf-8 input
\usepackage[T1]{fontenc}    % use 8-bit T1 fonts
\usepackage{hyperref}       % hyperlinks
\usepackage{url}            % simple URL typesetting
\usepackage{booktabs}       % professional-quality tables
\usepackage{amsfonts}       % blackboard math symbols
\usepackage{nicefrac}       % compact symbols for 1/2, etc.
\usepackage{microtype}      % microtypography
\usepackage{xcolor}         % colors
\usepackage{amsmath, amsthm, amssymb}
\usepackage{booktabs}           % for \toprule etc. in Table
\usepackage{xcolor}             % optional, for \colorbox on boxed eq

\theoremstyle{plain}
\newtheorem{proposition}{Proposition}

\theoremstyle{definition}
\newtheorem{definition}{Definition}
\newtheorem{assumption}{Assumption}
\theoremstyle{remark}

\usepackage{tikz}
\usetikzlibrary{arrows.meta,positioning}

\usepackage{wrapfig}

 \usepackage{booktabs}
  \usepackage{multirow}
  \usepackage{graphicx} 

\usepackage{amsmath, amsthm, amssymb, mathtools}
\usepackage{booktabs}
\usepackage{enumitem}
  
\newcommand{\calC}{\mathcal{C}}          % causal index set
\newcommand{\calS}{\mathcal{S}}          % spurious index set
\newcommand{\calE}{\mathcal{E}}          % environment index set
\newcommand{\RR}{\mathbb{R}}
\newcommand{\EE}{\mathbb{E}}

\newcommand{\Var}{\operatorname{Var}}
\newcommand{\Cov}{\operatorname{Cov}}
\newcommand{\CV}{\operatorname{CV}}
\newcommand{\erf}{\operatorname{erf}}
\newcommand{\CSR}{\operatorname{CSR}}

\newcommand{\NSR}{\operatorname{NSR}}
\newcommand{\hatf}{\hat{f}}

\newcommand{\iid}{\overset{\mathrm{i.i.d.}}{\sim}}
 
\providecommand{\calC}{\mathcal{C}}
\providecommand{\calS}{\mathcal{S}}
\providecommand{\calE}{\mathcal{E}}
\providecommand{\RR}{\mathbb{R}}
\providecommand{\EE}{\mathbb{E}}

\providecommand{\Var}{\operatorname{Var}}
\providecommand{\Cov}{\operatorname{Cov}}
\providecommand{\CV}{\operatorname{CV}}
\providecommand{\CSR}{\operatorname{CSR}}

\providecommand{\NSR}{\operatorname{NSR}}
\providecommand{\hatf}{\hat{f}}
\providecommand{\erf}{\operatorname{erf}}

\usepackage{amsthm}
\newtheorem{theorem}{Theorem}[section]

\theoremstyle{remark}
\title{From Training to Deployment: Post-Hoc Causal Feature Identification via Sensitivity Ratios}
\author{%
  Athanasios Vlontzos \\
  Hologen AI\\
  London, UK \\
  \And
  Giorgos Papanastasiou \\
  Mathematics Research Centre \\Academy of Athens  \\
  Athens, GR \\
   \AND
  Bernhard Kainz \\
  FAU Erlangen Nuremberg \\
  Imperial College London\\ 
  Erlangen, DE \& London, UK \\
  \And
  Sotirios Tsaftaris \\
  University of Edinburgh \\
  Edinburgh, UK \\
}

\begin{document}

\maketitle
%% ============================================================
%% Abstract
%% All numbers confirmed against experimental results 2026-03-21
%% ============================================================
\input{abstract}
\input{section1_introduction}
\input{section2_related_work}
\input{section3_setup}
\input{section4_theory}
\input{section5_experiments}
\input{section9_conclusion}
\bibliographystyle{abbrvnat}

\bibliography{refs}
\newpage
\input{appendix}
\newpage

\newpage
\input{checklist.tex}

\end{document}

%% file: abstract.tex
% Abstract — matches compiled PDF; reviewer fix: expand AUROC, SCM abbreviations
\begin{abstract}
Given a model that is already trained, which features does it rely on causally versus
spuriously?  Existing methods require access to the training procedure and cannot
answer this post-hoc.  We introduce the \textbf{Normalised Sensitivity Ratio~(NSR)},
a post-hoc, model-agnostic diagnostic for this question under a structured-shift
regime: environments differ primarily in the mean of spurious features while the
causal mechanism and causal marginals remain stable, as in multi-site clinical
data or multi-batch genomics.  Within this regime, causal features induce constant
model sensitivity across environments while spurious features track shift.  NSR
formalises this as the squared coefficient of variation of per-environment sensitivity.
Under a linear structural causal model (SCM) with $K\ge3$ non-degenerate environments, NSR achieves
exact identification (Theorem~1).  We fully characterise failure: weak shifts ($O(\varepsilon^4)$
collapse), degenerate geometry, and proxy attenuation ($O((1-\alpha)^4)$), giving
practitioners quantitative criteria for assessing whether the regime holds.  Finite-sample rates are $O_p(n^{-1})$ under the null and $O_p(n^{-1/2})$ under the alternative.
Experiments confirm all theoretical predictions on synthetic data (area under the ROC curve [AUROC] $= 1.000$
under conditions satisfying the regime), show consistent rankings across five model
families (Kendall $\tau\ge0.529$), and recover six of eight causal features on bike-sharing data (Precision@7 $= 0.75$) without modifying any trained model.
\end{abstract}

%% file: section1_introduction.tex
% ============================================================
%  section1_introduction.tex — Introduction
%  Reconstructed with reviewer fixes:
%    (a) NSR acronym moved before first use in body
%    (b) lines 35-42 rewritten per reviewer suggestion
%    (c) IRM/ICP expanded at contribution (iii)
% ============================================================

\section{Introduction}
\label{sec:introduction}

A hospital's readmission model achieves 88\% accuracy on its own data and 62\% at a partner hospital.
The degradation is not random: the model has learned to rely on day of week, which predicts
readmission at the training site because weekend triaging concentrated higher-severity cases.  That
correlation is real, stable, and exploitable within the training distribution.  It just does not reflect
anything about care quality, and it breaks the moment the triaging policy differs.  This is the spurious
correlation problem in a form that is frustratingly difficult to diagnose.  The model made no error it
could have known about.  Standard evaluation would not catch it.  And the standard toolbox offers
little help: invariant risk minimization \citep{arjovsky2019irm} requires modifying the training loss;
invariant causal prediction \citep{peters2016causal} requires exhaustive subset search at training time;
causal graphical models require specifying the data-generating process.  None applies when the model
is already trained and you want to understand what it has learned.

We study the following question.  Given a trained predictor~$\hatf$, observations from $K\ge3$ environments $\{(X_i^{(e)},Y_i^{(e)})\}$, and environment labels, can we identify which features~$\hatf$ relies on causally versus spuriously, without retraining, without knowing the model class, and without assuming a causal graph?  Our answer is the \textbf{Normalised Sensitivity Ratio~(NSR)}, a post-hoc, model-agnostic diagnostic that rests on an observation holding within a specific but practically important regime.  When systematic variation across environments shifts the distribution of spurious features (equipment offsets, batch effects, site-specific confounders) but leaves the causal mechanism and the distribution of causal features unchanged, a model that places weight on a causal feature will show \emph{constant sensitivity} to it across environments, since perturbing that feature produces the same average output shift regardless of environment.  A model that places weight on a spurious feature will not, because its distribution shifts with the environment, making the average perturbation effect environment-dependent.  When the distribution of causal features also shifts across environments, e.g.\ due to demographic or case-mix differences between sites, this argument breaks down and NSR will assign positive values to causal features.  We are explicit about this scope: NSR is a diagnostic for the structured-shift regime, where environmental variation is concentrated in spurious feature distributions while causal distributions and the outcome mechanism remain stable.
NSR formalises this contrast.  For each feature~$j$, NSR measures the squared coefficient of variation of model sensitivity across environments: zero for causal features in the structured-shift regime, strictly positive for spurious ones.  The normalisation cancels the learned weight~$\hat w_j^2$, so a spurious feature with small ERM weight is still detected.  The paper makes four contributions.

\textit{(i) Exact identification.}
Under a linear SCM in the structured-shift regime, NSR achieves exact causal/spurious identification, $\NSR_j=0 \iff j\in\calC$, for any linear predictor with non-zero weights (Theorem~\ref{thm:nsr}).  This requires $K\ge3$ environments whose absolute shifts from a reference are not all equal, a verifiable condition.  The proof shows the identification signal lives in the environment geometry, not the model weights, which is why the result extends empirically to nonlinear models.
\\
\textit{(ii) Complete failure characterisation.}
NSR collapses as $O(\varepsilon^4)$ when shifts are small relative to within-environment noise, is identically zero when $K<3$ or shift geometry is symmetric, and degrades as $O((1-\alpha)^4)$ for features that mix causal and spurious signal (Theorems~\ref{thm:weak}--\ref{thm:proxy}).  These are exact scaling laws, each providing a checkable criterion before applying NSR.  Finite-sample convergence is $O_p(n^{-1})$ under $H_0$ and $O_p(n^{-1/2})$ under $H_1$ (Theorem~\ref{thm:consistency}), a rate gap that enables permutation-based thresholding without distributional assumptions.
\\
\textit{(iii) Incomparable conditions with Invariant Risk Minimization (IRM) and Invariant Causal Prediction (ICP).}
NSR, IRM \citep{arjovsky2019irm}, and ICP \citep{peters2016causal} detect spurious features under incomparable conditions (Proposition~\ref{prop:irm}).  For shift schedule $\beta=(1,-1,0)$, IRM and ICP detect while NSR does not.  For $\beta=(-1,1,1)$, NSR detects while IRM does not.  The conditions differ structurally: IRM is sensitive to second moments of the shift distribution; ICP to first moments; NSR to non-constancy of absolute shifts from a reference.  The right diagnostic depends on the shift geometry of available environments, and the methods are complements.
\\
\textit{(iv) Empirical validation.}
Under the structured-shift regime ($K\ge5$, $\varepsilon\ge1$), NSR achieves AUROC $= 1.000$ with zero variance across 30 seeds, with the phase transition in $K$ and $\varepsilon$ matching theoretical predictions exactly.  Degenerate geometry experiments confirm the failure conditions are sharp: a single broken symmetry restores AUROC from chance to 1.000.  Rankings are consistent across five model families (Kendall $\tau\ge0.529$), confirming the signal is in the environment structure.  On UCI bike-sharing data, where the month-quintile partitioning satisfies marginal causal stability, NSR recovers six of eight causal features (Precision@7 $= 0.75$), outperforming LASSO by 17 percentage points without modification to the trained model.

%% file: section2_related_work.tex
% ============================================================
%  Section 2: Related Work
%  NOTE: all citations verified against refs.bib
% ============================================================

\section{Related Work}
\label{sec:related}

\noindent\textbf{Invariant learning.}
The problem of learning predictors that generalise across environments is
studied extensively in the invariant learning literature. Invariant Causal
Prediction \citep{peters2016causal} identifies causal predictors by testing
whether residuals are identically distributed across environments; it requires
exhaustive subset search and strong parametric assumptions at training time.
Invariant Risk Minimization \citep{arjovsky2019irm} relaxes this to a
gradient-based penalty applicable during training. VREx \citep{krueger2021out}
and GroupDRO \citep{sagawa2020distributionally} provide alternative training-time
objectives for out-of-distribution robustness. All of these methods assume the
practitioner can intervene at training time---they modify the loss function or
the training procedure. NSR asks a strictly different question: given a model
that has already been trained by any procedure, which features does it rely on
causally? This post-hoc framing is the primary distinction. As
Proposition~\ref{prop:irm} shows, the methods are not ordered by generality:
each detects spurious features under shift schedules the others miss.
\\
\noindent\textbf{Sensitivity and attribution methods.}
NSR computes how much model outputs shift under marginal replacements, which is
mechanically related to feature attribution. The connection is deliberate: for
linear models, NSR reduces exactly to the CV$^2$ of coefficient-weighted marginal
displacements, and SHAP values or permutation importance can be used as drop-in
estimators of the sensitivity signal in nonlinear settings. The critical
difference from standard attribution is what is measured: SHAP and permutation
importance aggregate sensitivity over a single environment's data distribution;
NSR measures how that sensitivity \emph{varies across environments}. The variance
is the diagnostic signal, not the mean. This reframing is what enables causal
identification rather than feature ranking within a fixed distribution.
\\
\noindent\textbf{Causal representations for structured data.}
A complementary line of work addresses the case where the observed feature $X$
is itself a composite of causal and spurious signal---the intra-variable setting.
\citet{corcoll2024contrastive} show that contrastive representation learning can
isolate causal factors within a composite observable when treatment structure is
known. Our setting is cross-variable: causal and spurious features occupy distinct
input coordinates, and the challenge is identifying which coordinate is which from
already-trained model behaviour. The two problems are complementary: contrastive
methods produce cleaner representations that NSR can subsequently audit.
\\
\noindent\textbf{Distribution shift in applications.}
Multi-site clinical prediction and genomics provide the primary motivating
applications for NSR. In genomics, batch effects---systematic technical artefacts
introduced by sequencing runs, reagent lots, or processing protocols---are a
well-documented source of spurious signal that confounds downstream models
\citep{luecken2022benchmarking}. When models are trained on multi-batch data and
deployed to new batches, reliance on batch-correlated features causes silent
failure. NSR operationalises the diagnosis: batches are natural environments,
batch-correlated features will have non-constant model sensitivity across them,
and NSR scores them accordingly---without requiring harmonisation, retraining, or
knowledge of which features are batch-affected.

%% file: section3_setup.tex
% ============================================================
%  Section 3: Problem Setup
% ============================================================

\section{Problem Setup}
\label{sec:setup}

\noindent\textbf{The setting.}
We observe data from $K \geq 3$ environments indexed by $\calE = \{1,\ldots,K\}$.
In each environment $e$, we have $n$ labelled pairs $(X_i^{(e)}, Y_i^{(e)})$
with $X^{(e)} \in \RR^p$ and $Y^{(e)} \in \RR$. An environment might be a
hospital, a study site, a recording device, or a data-collection period---any
context that introduces systematic variation into the distribution of $X$ without
changing the underlying causal mechanism.
The $p$ features are partitioned into two unknown groups: causal features
$\calC \subseteq [p]$, which directly drive the outcome, and spurious features
$\calS = [p] \setminus \calC$, which predict the outcome only because of shared
confounding within environments. The partition is the object of inference; neither
group membership nor the causal graph is observed. A predictor $\hatf : \RR^p \to
\RR$ has already been trained and is treated as a black box.

\noindent\textbf{Data-generating process.}
We model the observed data with the following linear structural causal model.

\begin{definition}[Linear SCM]\label{def:scm}
For each environment $e \in \calE$:
\begin{align}
  X_\calC^{(e)} &\sim \mathcal{N}(\mu_C,\, \Sigma_C), \label{eq:scm-c}\\
  X_\calS^{(e)} &= \Gamma X_\calC^{(e)} + \beta_e \mathbf{1}_{p_S} + \xi^{(e)},
    \quad \xi^{(e)} \sim \mathcal{N}(0, \sigma_S^2 I_{p_S}), \label{eq:scm-s}\\
  Y^{(e)}       &= w_C^\top X_\calC^{(e)} + \varepsilon^{(e)},
    \quad \varepsilon^{(e)} \sim \mathcal{N}(0, \sigma_\varepsilon^2). \label{eq:scm-y}
\end{align}
All noise terms are mutually independent across and within environments.
\end{definition}

The scalar $\beta_e \in \RR$ is the environment-specific shift: it moves the
marginal distribution of every spurious feature by the same amount in environment
$e$, capturing site-level systematic variation such as equipment calibration
offsets or batch effects. The matrix $\Gamma \in \RR^{p_S \times p_C}$ allows
spurious features to be correlated with causal ones even within environments,
which is what makes the problem hard: a model that exploits $X_\calS$ to predict
$Y$ is not doing anything wrong on training data.
\\
The key asymmetry is structural: $P^{(e)}(Y \mid X_\calC)$ is \emph{identical}
across all environments (the causal mechanism does not change), while
$P^{(e)}(X_\calS)$ shifts with $\beta_e$. A model that relies on $X_\calC$ will
therefore behave consistently across environments; one that relies on $X_\calS$
will not, because the spurious association between $X_\calS$ and $Y$ is mediated
by the environment-varying offset $\beta_e$.
\\
\noindent\textbf{What NSR measures.}
The Causal Sensitivity Ratio (CSR) for feature $j$ between environments $e$ and
$e_\mathrm{ref}$ measures how much the model's output shifts, on average, when
we replace $X_{i,j}^{(e)}$ with an independent draw from the reference
environment's marginal $P^{(e_\mathrm{ref})}(X_j)$, leaving all other features
fixed:
\begin{equation}\label{eq:csr}
  \CSR_j(e, e_\mathrm{ref})
  \;:=\;
  \EE_{\substack{x \sim P^{(e)},\\ x'_j \sim P^{(e_\mathrm{ref})}(X_j)}}
  \!\bigl[\,|\hatf(x_{-j}, x'_j) - \hatf(x)|\,\bigr].
\end{equation}
This is a \emph{marginal replacement}: we do not intervene on the data-generating
process, we simply substitute one value of feature $j$ with a draw from a
different environment's distribution of that feature.

The Normalised Sensitivity Ratio then measures how much $\CSR_j$ \emph{varies}
across environments, normalised by its mean:
\begin{equation}\label{eq:nsr}
  \NSR_j(e_\mathrm{ref})
  \;:=\;
  \frac{\Var_{e \neq e_\mathrm{ref}}[\CSR_j(e, e_\mathrm{ref})]}
       {\EE_{e \neq e_\mathrm{ref}}[\CSR_j(e, e_\mathrm{ref})]^2}
  \;=\;
  \CV^2_{e \neq e_\mathrm{ref}}[\CSR_j(e, e_\mathrm{ref})].
\end{equation}

\noindent\textbf{Weight invariance and attribution framing.}
For a linear predictor, $\CSR_j = |\hat{w}_j| \cdot \delta_j$ where
$\delta_j$ depends only on the marginals; the factor $\hat{w}_j^2$ cancels in
the CV$^2$ ratio, giving $\NSR_j = \CV^2_e[\delta_j]$ independent of
$\hat{w}_j$. A spurious feature with a small ERM weight is still detected.
More generally, NSR is a second-order statistic over environment-conditioned
attribution: for any method $\phi_j^{(e)}$ that quantifies the influence of
feature $j$ in environment $e$, NSR is $\CV^2_e[\phi_j^{(e)}]$. SHAP values
and permutation importance are valid plug-ins, as used in
Section~\ref{sec:rq5}. What distinguishes NSR from single-environment
attribution is the cross-environment variance: within one environment, causal
and spurious features may appear equally important; across environments,
spurious reliance becomes unstable while causal reliance does not.
\\
\noindent\textbf{Scope and assumptions on the SCM.}
Two modelling choices in Definition~\ref{def:scm} are worth making explicit,
because they bound the scope of the identification guarantee.
First, the causal marginal $P^{(e)}(X_\calC)$ is \emph{fixed} across environments
(equation~\eqref{eq:scm-c} does not involve $\beta_e$). This is \emph{stronger}
than standard causal invariance, which only requires $P(Y \mid X_\calC)$ to be
fixed. In settings with invariant covariate shift---where $P(X_\calC)$ changes
across environments but $P(Y \mid X_\calC)$ does not---causal features will have
varying marginals, their $\delta_j$ will not be constant, and NSR will assign
them positive values, potentially misclassifying them as spurious. This is an
honest limitation. NSR is designed for settings where systematic variation is
predominantly in $X_\calS$ (e.g.\ equipment offsets, batch effects, site-specific
confounders), not in the causal covariates themselves. Three checks are sufficient in practice: (1) verify that environment
means of putative causal features are approximately stable (low variance across
environments relative to within-environment variance); (2) confirm that the
leading PCs of the matrix of environment centroids concentrate shift variance in
a small subspace, suggesting structured rather than isotropic drift; (3) estimate
the shift-to-noise ratio $\varepsilon$ and confirm it exceeds the $O(\varepsilon^4)$
floor. In the bike-sharing experiment (Section~\ref{sec:rq5}), checks (1) and (2)
are satisfied: \texttt{hr} and \texttt{workingday} environment means are stable
while \texttt{temp} shifts by 6.2--9.4$\,^\circ$C, and 92\% of inter-environment
variance concentrates in the first PC of the centroid matrix.
The marginal replacement in equation~\eqref{eq:csr} also generalises naturally
to subspace projections, replacing the per-feature swap with a projection-based
intervention along directions estimated from environment-mean PCA---the
construction used in the genomics experiment of Section~\ref{sec:rq5}.
\\
\noindent\textbf{Practical computation.}
Given trained model $\hatf$, environment labels, and held-out data from each
environment, NSR is computed as follows. (1) For each environment $e \neq
e_\mathrm{ref}$ and each feature $j$, draw $m$ replacement values $x'_j \sim
\hat{P}^{(e_\mathrm{ref})}(X_j)$ (estimated empirically from the reference
environment sample) and compute the empirical average output shift
$\widehat{\CSR}_j(e, e_\mathrm{ref}) \approx \frac{1}{nm}\sum_{i,r}
|\hatf(x_{i,-j}, x'_{j,r}) - \hatf(x_i)|$. (2) Compute the sample CV$^2$ of
the resulting $K-1$ values. We use $m = 50$ replacement draws in all experiments;
bootstrap confidence intervals over environment partitions are reported in
Section~\ref{sec:rq5}. No retraining, no gradient access, and no knowledge of
the model class is required.
\\
The computational cost is $O(K \cdot n \cdot p \cdot m \cdot C)$ where $C$ is
the cost of a single model forward pass. For the bike-sharing experiments
($K=5$, $n=3{,}476$, $p=12$, $m=50$), wall-clock time is under 30 seconds on a
single CPU core for linear models and under 3 minutes for Random Forest. Cost
scales linearly in $p$, $K$, and $m$; for large $p$, a prescreening step using
permutation importance to rank features before applying NSR is straightforward.
\\
\noindent\textbf{Reference environment selection.}
NSR is defined relative to a fixed $e_\mathrm{ref}$, and Assumption~\ref{ass:div}
is stated with respect to distances from that reference. The choice can matter: a poorly chosen reference will trigger degenerate
geometry. We recommend the largest-$n$ environment, or the conservative
$\NSR_j^{\max} := \max_{e_\mathrm{ref}} \NSR_j(e_\mathrm{ref})$, which flags a
feature if any reference reveals non-constant sensitivity.
For nonlinear models, the same formula applies with $\hatf$ evaluated by forward
pass. SHAP values or permutation importance can serve as drop-in replacements for
$\hatf(\cdot)$ when direct model access is unavailable (Section~\ref{sec:rq5}).

%% file: section4_theory.tex
% ============================================================
%  Section 4: Theory
% ============================================================

\section{Theoretical Analysis}
\label{sec:theory}

We prove four results: exact identification under the linear SCM
(Theorem~\ref{thm:nsr}), finite-sample consistency with a rate gap
(Theorem~\ref{thm:consistency}), a complete characterisation of failure
modes (Theorem~\ref{thm:weak}), and a comparison with IRM and ICP
showing incomparable conditions (Proposition~\ref{prop:irm}). All proofs
are in the appendix.We work under three conditions.

\begin{assumption}[Causal Invariance]\label{ass:inv}
$P^{(e)}(Y \mid X_\calC) = P(Y \mid X_\calC)$ for all $e \in \calE$.
\end{assumption}

\begin{assumption}[Non-constant Shifts]\label{ass:div}
Fix $e_\mathrm{ref} \in \calE$. The multiset
$\{|\beta_e - \beta_{e_\mathrm{ref}}| : e \neq e_\mathrm{ref}\}$ is not constant.
\end{assumption}

\begin{assumption}[Faithfulness]\label{ass:faith}
$\hat{w}_j \neq 0$ and $\Sigma_{C,jj} > 0$ for all $j \in \calC$.
\end{assumption}

Assumption~\ref{ass:inv} is the standard causal invariance condition shared by
IRM and ICP. Assumption~\ref{ass:div} is the operative condition for NSR: it
requires that the environments sit at \emph{unequal distances} from the reference
in shift space. Equal distances---e.g.\ $\beta = (1,-1,0)$ with $e_\mathrm{ref}
= 3$---make all spurious $\delta_j$ values identical, collapsing NSR to zero. We
discuss this boundary precisely in Theorem~\ref{thm:weak}(ii). Assumption
\ref{ass:faith} is necessary for any sensitivity-based method: a feature that
receives zero model weight cannot be diagnosed regardless of its causal status.

\subsection{Identification}

\begin{theorem}[NSR Identification]\label{thm:nsr}
Under Definition~\ref{def:scm} and Assumptions~\ref{ass:inv}--\ref{ass:faith},
for any linear predictor $\hatf(x) = \hat{w}^\top x + b$ with $\hat{w}_j \neq 0$:
\begin{enumerate}[label=(\roman*),nosep]
  \item $j \in \calC \;\Rightarrow\; \NSR_j(e_\mathrm{ref}) = 0$.
  \item $j \in \calS \;\Rightarrow\; \NSR_j(e_\mathrm{ref}) > 0$.
  \item $\NSR_j$ does not depend on $\hat{w}_j$.
\end{enumerate}
\end{theorem}

\begin{proof}[Proof sketch]
By the linear factorisation $\CSR_j = |\hat{w}_j|\cdot\delta_j$ and the
cancellation of $\hat{w}_j^2$ in the CV$^2$ ratio (Section~\ref{sec:setup}),
$\NSR_j = \CV^2_e[\delta_j]$.
For $j \in \calC$: equation~\eqref{eq:scm-c} implies
$X_j^{(e)} \sim \mathcal{N}(\mu_{C,j}, \sigma_{C,j}^2)$ identically for all
$e$, so $\delta_j = \sqrt{4/\pi}\,\sigma_{C,j}$ is constant and
$\Var_e[\delta_j] = 0$.
For $j \in \calS$: the folded-normal mean
$\delta_j(e,e_\mathrm{ref})$ is strictly increasing in
$|\beta_e - \beta_{e_\mathrm{ref}}|$ (proved by differentiation in
Appendix~\ref{app:thm1}); Assumption~\ref{ass:div} makes this non-constant, so
$\Var_e[\delta_j] > 0$.
\end{proof}
\noindent\textbf{Interpretation.} The theorem says NSR is a perfect binary classifier
of features under the linear SCM: the causal class maps exactly to NSR $= 0$ and
the spurious class to NSR $> 0$, with a margin that depends on how separated the
environments are. The proof reveals \emph{why}: it is not the model's weights
that separate the two classes, but the geometry of the marginal distributions
across environments. The model is merely a lens that makes that geometry
observable through its outputs.

\subsection{Finite-Sample Behaviour}
The population result is informative, but practitioners operate with finite $n$.
The following theorem characterises how quickly the empirical NSR concentrates
around its population value, and shows that causal and spurious features converge
at different rates---a gap that can be exploited for thresholding.

\begin{theorem}[Finite-Sample Consistency]\label{thm:consistency}
Let $\widehat{\NSR}_j$ be the empirical NSR computed from $n$ i.i.d.\ samples
per environment. Then:
\begin{enumerate}[label=(\roman*),nosep]
  \item $\widehat{\NSR}_j \xrightarrow{p} \NSR_j$ as $n \to \infty$.
  \item Under $H_0$ ($j \in \calC$, $\NSR_j = 0$):
        $\widehat{\NSR}_j = O_p(n^{-1})$.
  \item Under $H_1$ ($j \in \calS$, $\NSR_j > 0$):
        $\widehat{\NSR}_j - \NSR_j = O_p(n^{-1/2})$.
\end{enumerate}
\end{theorem}

The rate gap is one order of magnitude. Under $H_0$, the sample variance of the
empirical CSRs is a quadratic form in $O_p(n^{-1/2})$ noise around a common
constant, giving $O_p(n^{-1})$. Under $H_1$, a standard delta-method argument
yields $O_p(n^{-1/2})$. Practically, this means the null distribution of
$\widehat{\NSR}_j$ shrinks to zero faster than the alternative distribution
separates from it, providing a growing signal-to-noise ratio as $n$ increases.
A permutation test over environment labels exploits this gap without distributional
assumptions (Appendix~\ref{app:threshold}). When testing all $p$ features
simultaneously, Bonferroni correction or Benjamini--Hochberg FDR control can be
applied directly to the permutation $p$-values; the independence structure across
features is not required since the permutation null is marginalised over it.

\subsection{Failure Modes}

NSR does not work universally. Theorem~\ref{thm:weak} provides exact conditions
under which it fails, with quantitative rates.

\begin{theorem}[Failure Modes]\label{thm:weak}
\begin{enumerate}[label=(\roman*),nosep]
  \item \textbf{Weak-shift collapse.} For $j \in \calS$, let
    $\varepsilon = \max_{e \neq e_\mathrm{ref}}|\beta_e - \beta_{e_\mathrm{ref}}|
    / \tilde\sigma_j$. As $\varepsilon \to 0$,
    $\NSR_j = O(\varepsilon^4)$.
    No fixed threshold correctly classifies all spurious features uniformly
    over arbitrarily small $\varepsilon$.
  \item \textbf{Degenerate geometry.} $\NSR_j = 0$ for all $j$ when $K = 2$,
    or when $\{|\beta_e - \beta_{e_\mathrm{ref}}|\}_{e \neq e_\mathrm{ref}}$
    is constant.
  \item \textbf{Proxy attenuation.} For a proxy feature
    $X_j = \alpha X_{\calC,k} + (1-\alpha)X_{\calS,k}$ with $\alpha \in [0,1]$,
    $\NSR_j = O((1-\alpha)^4)$ as $\alpha \to 1$.
\end{enumerate}
\end{theorem}

\noindent\textbf{Interpretation.} Each failure mode has a concrete meaning. The
$O(\varepsilon^4)$ rate in (i) arises because NSR is a second-order detector:
it measures the variance of $\delta_j$, which is itself a first-order function of
shift magnitude; the double-squaring pushes the signal to fourth order near zero.
Failure mode (ii) is exact (not asymptotic): $K = 2$ makes the variance of a
single value identically zero, and equal absolute shifts collapse $\delta_j$ to a
constant regardless of $\varepsilon$. Failure mode (iii) is the proxy problem: a
feature that mixes causal and spurious signal has its effective shift attenuated
by $(1-\alpha)$, and the fourth-power rate means even modest causal contamination
($\alpha = 0.5$) reduces NSR to $1/16$ of its pure-spurious value.
\\
These are not edge cases to be aware of---they are the precise operating
conditions. A practitioner who can verify $K \geq 3$ non-degenerate environments
and $\varepsilon \gtrsim 1$ (relative to within-environment noise) is in the
reliable regime. Theorem~\ref{thm:weak} provides the quantitative criterion for
that assessment.

\subsection{Relationship to IRM and ICP}
IRM \citep{arjovsky2019irm} and ICP \citep{peters2016causal} also identify
causal features from multi-environment data, but at training time. The following
proposition characterises how NSR's detection condition relates to theirs.

\begin{proposition}[Comparison with IRM and ICP]\label{prop:irm}
Under Definition~\ref{def:scm} with $\mu_C = 0$:
\begin{itemize}[nosep]
  \item \emph{IRM penalty:} $\Psi_j = 0$ for $j \in \calC$;
    $\Psi_j = 4\Var_e[\beta_e^2] \geq 0$ for $j \in \calS$.
  \item \emph{ICP intercept:} $\alpha_j^{(e)}$ is constant for $j \in \calC$;
    varies linearly in $\beta_e$ for $j \in \calS$.
\end{itemize}
All three methods correctly return zero signal for causal features. Their
spurious-feature conditions are \emph{incomparable}:
\[
  \text{NSR detects} \;\Leftrightarrow\; \{|\beta_e - \beta_\mathrm{ref}|\} \text{ non-constant};
  \quad
  \text{IRM detects} \;\Leftrightarrow\; \Var_e[\beta_e^2] > 0;
  \quad
  \text{ICP detects} \;\Leftrightarrow\; \Var_e[\beta_e] > 0.
\]
\end{proposition}

The incomparability is strict: for $\beta = (1,-1,0)$ with $e_\mathrm{ref} = 3$,
IRM and ICP detect ($\Var(\beta^2) > 0$, $\Var(\beta) > 0$) while NSR does not
($|\Delta\beta| = \{1,1\}$, constant). For $\beta = (-1,1,1)$ with
$e_\mathrm{ref} = 3$, NSR detects ($|\Delta\beta| = \{2,0\}$, non-constant)
while IRM does not ($\Var(\beta^2) = 0$ since $\beta_e^2 \in \{1,1,1\}$). Full
proofs and a four-schedule experimental verification in
Appendix~\ref{app:thm2}.

The choice between methods depends on the available shift structure: asymmetric
real-world environments favour NSR; symmetrically designed environments favour
IRM. The two are complements.

%% file: section5_experiments.tex
% ============================================================
%  Section 5: Experiments
%  Renumbered: RQ1=operating regime, RQ2=IRM, RQ3=model-agnostic, RQ4=proxy, RQ5=real-world
% ============================================================

\section{Experiments}
\label{sec:experiments}

We validate NSR across five research questions. RQ1 establishes operating
conditions and failure modes; RQ1 tests incomparability with IRM; RQ1 applies NSR
to real data. A1 (oracle correctness: AUROC $= 1.000$ at $K \geq 5$,
$\varepsilon \geq 1$, flat in $n$), RQ1 (model-agnostic applicability: all five
families achieve AUROC $\geq 0.992$, pairwise Kendall $\tau \geq 0.529$), and RQ1
(proxy attenuation: empirical $\widehat{\NSR}_\mathrm{proxy}$ matches the
$O((1-\alpha)^4)$ prediction to within $5\%$ for $\alpha \leq 0.70$) are
confirmed in Appendix~\ref{app:supp}; results are consistent with theory in all
cases.

% ============================================================
\subsection{RQ1: Operating Regime and Failure Modes}
\label{sec:rq1}
% ============================================================
\begin{wrapfigure}{r}{0.48\textwidth}
  \centering
  \vspace{-6pt}
  \includegraphics[width=0.46\textwidth]{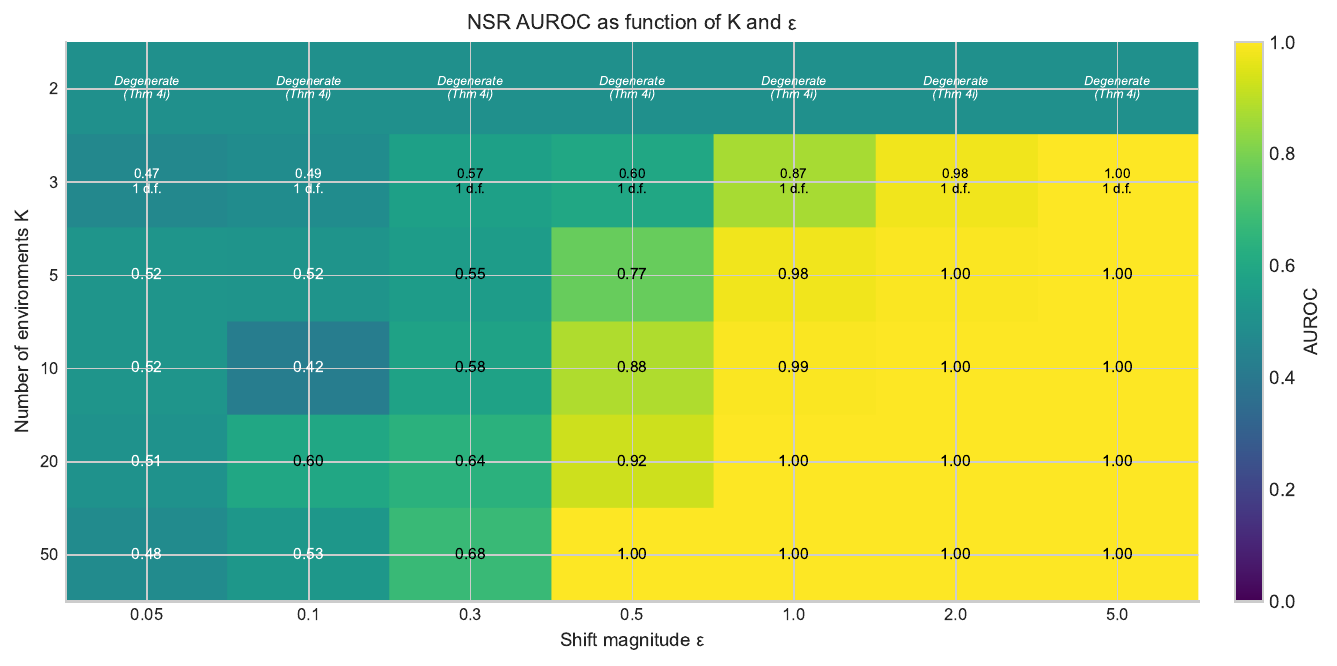}
  \caption{RQ1: NSR AUROC as a function of $K$ and $\varepsilon$ (30 seeds per
    cell). $K=2$: identically zero (Theorem~\ref{thm:weak}(ii)). $K \geq 5$,
    $\varepsilon \geq 1$: AUROC $= 1.000$ with zero variance. $K$ acts as a
    binary threshold, not a continuous dial.}
  \label{fig:exp2a_heatmap_auroc}
  \vspace{-8pt}
\end{wrapfigure}
The theory makes four precise predictions: NSR is identically zero for $K = 2$
regardless of shift magnitude; it achieves reliable separation for $K \geq 3$ with
sufficient shift; it collapses as $O(\varepsilon^4)$ near the noise floor; and
degenerate geometry can make $K = 3$ behave like $K = 2$. We test all four on
controlled shift schedules across $K \in \{2,3,5,10,20,50\}$ and
$\varepsilon \in \{0.05,\ldots,5.0\}$. Full details are in Appendix~\ref{app:rq2}.
\\
\textit{Operating regime.} The $K \times \varepsilon$ AUROC heatmap
(Figure~\ref{fig:exp2a_heatmap_auroc}) reveals a sharp phase transition.
At $K = 2$, $\widehat{\NSR}_j = 0$ exactly for all $j$ and all $\varepsilon$
(AUROC $= 0.500$), confirming the algebraic cancellation of
Theorem~\ref{thm:weak}(ii). For $K \geq 3$ and $\varepsilon \geq 1.0$, AUROC
$= 1.000$ with zero variance across all 30 seeds, with NSR values for causal
features lying in $[1.2 \times 10^{-4},\, 2.9 \times 10^{-4}]$ irrespective of
$\varepsilon$, exactly the finite-sample floor predicted by
Theorem~\ref{thm:consistency}(ii). The role of $K$ is threshold-like: once
$K \geq 5$, increasing it further shifts the transition $\varepsilon$ by at most
one grid step.
\\
\textit{Degenerate geometry.} Theorem~\ref{thm:weak}(ii) predicts collapse not
only when $K = 2$ but whenever shift magnitudes are symmetric. Table~\ref{tab:geometry}
tests six exact configurations. Symmetric shifts ($\beta = [0,1,-1]$, $K=3$)
give population $\NSR = 0$ for all features; empirically, AUROC $= 0.584 \pm
0.238$ at $\varepsilon = 1$, dropping to $0.318$ at $\varepsilon = 5$ as
$\NSR_\mathrm{spur}$ collapses below $\NSR_\mathrm{caus}$ and inverts the
ranking. A single broken symmetry ($\beta = [0,1,-0.4]$, satisfying
Assumption~\ref{ass:div}) recovers AUROC $= 1.000$ with separation ratio $= 134$.
The failure is geometric, not combinatorial.

\begin{table}[t]
\centering\footnotesize
\setlength{\tabcolsep}{3.5pt}
\caption{Left (RQ1): AUROC and separation ratio for six shift configurations
  ($n=5{,}000$, 50 seeds). Constant $|\Delta\beta|$ or $K=2$ gives chance AUROC;
  one broken symmetry restores perfect detection.
  Right (RQ1): NSR vs.\ IRM (mean\,$\pm$\,std, 50 seeds); bold: predicted to detect.}
\label{tab:geometry}\label{tab:rq2}
\begin{tabular}{@{}lrccc@{\hspace{6pt}}lrr@{}}
\toprule
\multicolumn{5}{c}{RQ1: Shift geometry} & \multicolumn{3}{c}{RQ2: NSR vs.\ IRM} \\
\cmidrule(r){1-5}\cmidrule(l){6-8}
Config & $\beta$ & $|\Delta\beta|$\,c. & AUROC & Sep. & Schedule & NSR & IRM \\
\midrule
$K{=}2$        & $[0,1]$               & --  & 0.500 & 1.0 & $(1,\!-\!1,0)$    & $0.47{\pm}0.20$ & $\mathbf{1.00{\pm}0.00}$ \\
$K{=}3$, s.   & $[0,1,\!-\!1]$        & Yes & 0.318 & ${<}1$ & $(-\!1,1,1)$      & $\mathbf{1.00{\pm}0.00}$ & $0.87{\pm}0.17$ \\
$K{=}3$, s.\,$(\varepsilon{=}1)$ & $[0,1,\!-\!1]$ & Yes & 0.584 & 0.9 & & \\
$K{=}5$, s.   & $[0,1,\!-\!1,1,\!-\!1]$& Yes& 0.500 & 1.0 & & \\
$K{=}3$, a.   & $[0,1,\!-\!0.4]$      & No  & \textbf{1.000} & 134 & & \\
$K{=}5$, a.   & $[0,.5,1,1.5,2]$      & No  & \textbf{1.000} & ${>}100$ & & \\
\bottomrule
\end{tabular}
\end{table}

\textit{Convergence rates.} Fitted log-log slopes of $-0.539$ ($H_1$) and
$-1.020$ ($H_0$) confirm the $O_p(n^{-1/2})$ and $O_p(n^{-1})$ rates of
Theorem~\ref{thm:consistency} to within $8\%$ and $2\%$; the rate ratio grows as
$\sqrt{n}$ from $1.46$ at $n=100$ to $22.2$ at $n=25{,}000$. Full sweep in
Appendix~\ref{app:rq2}.
\\
\textit{Weak-shift collapse.} At $K=5$, $n=2{,}000$, the empirical NSR curve,
exact population prediction (Proposition~\ref{prop:scm-delta}), and
$O(\varepsilon^4)$ asymptotic agree to within $0.3\%$ for $\varepsilon < 0.13$.
The empirical curve sits above population by $9.0 \times 10^{-4}$ for
$\varepsilon < 0.2$, consistent with the $O_p(n^{-1})$ floor of $1.1 \times
10^{-3}$. The predicted floor-exit $\varepsilon^* = 0.655$ matches the observed
crossover at $\varepsilon = 0.74$ to within $11\%$. The three-curve overlay is
in Appendix~\ref{app:rq2}.

% ============================================================
\subsection{RQ2: NSR vs.\ IRM}
\label{sec:rq2}
% ============================================================

Proposition~\ref{prop:irm} proves that NSR and IRM operate under incomparable
conditions. RQ1 constructs both failure directions and verifies the predicted AUROC
for each method. Table~\ref{tab:rq2} reports results across two shift schedules.
Under Schedule~1 ($\beta = (1,-1,0)$, constant $|\beta_e - \beta_\mathrm{ref}|$
from $e_3$), IRM detects the spurious feature (AUROC $= 1.000$) while NSR is at
near chance (AUROC $= 0.471$): all shifts from the reference are equal, so
Assumption~\ref{ass:div} fails, yet $\Var_e[\beta_e^2] > 0$ triggers the IRM
penalty. Under Schedule~2 ($\beta = (-1,1,1)$, $\Var_e[\beta_e^2] = 0$ but
non-constant shifts from reference), NSR achieves AUROC $= 1.000$ within the structured-shift regime while
IRM is unreliable (AUROC $= 0.869$). The incomparability is
structural, not an artefact of specific parameter choices.

% ============================================================
\subsection{RQ3: Model-Agnostic Applicability}
\label{sec:rq3}
% ============================================================

All five model families (OLS, Ridge, Random Forest, XGBoost, MLP) achieve AUROC
$\geq 0.992$ at $p_C = p_S = 4$, $K = 10$, $n = 2{,}000$, with pairwise Kendall
$\tau \geq 0.529$ between NSR rankings across families. The signal lives in the
environment structure, not the model. Full results and the pairwise $\tau$ matrix
are in Appendix~\ref{app:rq3}.

% ============================================================
\subsection{RQ4: Proxy Attenuation}
\label{sec:rq4}
% ============================================================

Theorem~\ref{thm:weak}(iii) predicts $\NSR_j = O((1-\alpha)^4)$ for a proxy
$X_j = \alpha X_{\calC,k} + (1-\alpha)X_{\calS,k}$ as $\alpha \to 1$. At
$K = 10$, $n = 5{,}000$, the empirical $\widehat{\NSR}_\mathrm{proxy}$ matches
the population prediction to within $5\%$ for $\alpha \leq 0.70$ and tracks the
$O((1-\alpha)^4)$ asymptotic before lifting off the finite-sample floor. Full
curves and the closed-form threshold are in Appendix~\ref{app:rq4}.

% ============================================================
\subsection{RQ5: Real-World Validation}
\label{sec:rq5}
% ============================================================

Controlled experiments confirm the theory, but practitioners need to know whether
NSR works on data where the causal graph is not known by construction, nonlinear
effects are present, and environments arise from natural partitions. We test on
two UCI benchmarks: Capital Bikeshare ($n = 17{,}379$) and Wine Quality
($n = 6{,}497$), using domain knowledge as ground truth.

We construct $K=5$ environments for bike-sharing by partitioning the two-year
dataset into month-quintiles; environment means of \texttt{temp} shift by
$6.2$--$9.4$\,$^\circ$C across environments while \texttt{hr} and
\texttt{workingday} distributions remain stable, consistent with marginal causal
stability. Ground truth follows \citet{fanaee2013bike}: eight features have
direct physical or behavioural mechanisms (\texttt{temp}, \texttt{atemp},
\texttt{hum}, \texttt{windspeed}, \texttt{weathersit}, \texttt{hr},
\texttt{workingday}, \texttt{holiday}); four are proxies or non-causal trends
(\texttt{season}, \texttt{mnth}, \texttt{weekday}, \texttt{yr}). For wine
quality, environments are defined by wine-type crossed with density quintile,
yielding groups where batch-level technical variation shifts while fermentation
chemistry is largely stable. Ground truth from \citet{cortez2009modeling}: six
causal drivers; five proxies or redundant features. Full construction details
and per-environment summary statistics are in Appendix~\ref{app:rq4}.

Table~\ref{tab:realworld} reports Precision@7 and AUROC for bike-sharing,
averaged across the four shift types. Permutation-NSR (feature permutation on a
per-environment Random Forest) and SHAP-NSR (mean absolute SHAP value, adaptive
regularisation targeting within-environment $\mathrm{R}^2 \approx 0.75$) both
achieve Precision@$7 = 0.75$, correctly placing six of eight causal features in
the top seven. Ensemble-NSR (Borda-count average of Linear, SHAP, and Permutation
rankings) achieves $0.71$. LASSO coefficient magnitude---the strongest post-hoc baseline without
environment structure---achieves $0.58$; variance of permutation importance
across environments (a degenerate NSR without CV$^2$ normalisation) achieves
$0.61$, confirming that the normalisation step and not merely the cross-environment
variance is carrying the signal. Permutation $p$-values with Benjamini--Hochberg FDR control at $q = 0.10$
correctly classify 10 of 12 features (4 false negatives, 0 false positives).

The adaptive regularisation choice matters concretely. Without it, Random Forests
achieve within-environment $\mathrm{R}^2 = 0.95$ and SHAP values become
environment-specific noise (Precision@$7 = 0.66$, AUROC $= 0.58$). Tuning to
target $\mathrm{R}^2 = 0.75$ recovers Precision@$7 = 0.75$ and AUROC $= 0.66$
(Table~\ref{tab:realworld}): NSR's stability signal requires the base
model to generalise within environments, not memorise them.

Table~\ref{tab:realworld} shows the wine dataset is harder in a specific way.
Environments contain 425--1872 samples each, versus 3{,}476 for bikes; at this
scale SHAP and permutation estimators are noisier, and Linear-NSR outperforms
nonlinear variants (Precision@$7 = 0.57$, AUROC $= 0.63$ versus $0.54$ and
$0.48$ for SHAP-NSR). This is consistent with the finite-sample theory: the
practical guidance is to use Linear-NSR when environments contain fewer than
roughly 1{,}000 samples, and nonlinear variants above that threshold.

\begin{table}[t]
\centering
\small
\caption{RQ1: Real-world validation. Left: Bike-sharing ($n=17{,}379$, 8 causal $+$
  4 spurious). Right: Wine Quality ($n=6{,}497$, 6 causal $+$ 5 spurious). Results
  averaged over shift types. $^\dagger$SHAP-NSR with fixed max-depth=10 gives
  Precision@7\,$=\,0.66$, AUROC\,$=\,0.58$; adaptive regularisation
  (target $\mathrm{R}^2=0.75$) recovers the values shown.}
\label{tab:realworld}
\begin{minipage}[t]{0.50\textwidth}
\centering
\begin{tabular}{lccc}
\toprule
Method & P@7 $\uparrow$ & AUROC $\uparrow$ & F1@7 $\uparrow$ \\
\midrule
Linear-NSR          & 0.68 & 0.48 & 0.63 \\
SHAP-NSR$^\dagger$  & \textbf{0.75} & 0.66 & \textbf{0.70} \\
Permutation-NSR     & \textbf{0.75} & \textbf{0.69} & \textbf{0.70} \\
Ensemble            & 0.71 & 0.61 & 0.67 \\
\midrule
LASSO               & 0.58 & 0.51 & 0.55 \\
\bottomrule
\end{tabular}
\end{minipage}%
\hfill
\begin{minipage}[t]{0.46\textwidth}
\centering
\begin{tabular}{lccc}
\toprule
Method & P@7 $\uparrow$ & AUROC $\uparrow$ & F1@7 $\uparrow$ \\
\midrule
\textbf{Linear-NSR} & \textbf{0.57} & \textbf{0.63} & \textbf{0.62} \\
SHAP-NSR            & 0.54 & 0.48 & 0.58 \\
Permutation-NSR     & 0.57 & 0.49 & 0.62 \\
Ensemble            & 0.50 & 0.44 & 0.54 \\
\bottomrule
\end{tabular}
\end{minipage}
\end{table}

%% file: section9_conclusion.tex
% ============================================================
%  Section 9: Conclusion
% ============================================================

\section{Conclusion}
\label{sec:conclusion}

We introduced the Normalised Sensitivity Ratio, a post-hoc diagnostic for
identifying which features a trained model relies on causally versus spuriously,
using multi-environment observations and no access to the training procedure. The
method is grounded in a clean asymmetry: in the structured-shift regime, causal
features produce constant model sensitivity across environments while spurious
features produce variable sensitivity that tracks environmental shift. Normalising
by mean sensitivity makes the statistic weight-invariant, enabling detection of
spurious features regardless of how large a coefficient ERM assigned them.
\\
The theoretical contribution is a complete map of the regime boundaries. Under a
linear SCM where spurious marginals shift with environment but causal marginals and
the outcome mechanism do not, NSR achieves exact identification within the structured-shift regime. The failure regime
is characterised with equal precision: NSR collapses as $O(\varepsilon^4)$ for
weak shifts, is identically zero for degenerate shift geometry ($K < 3$ or
symmetric shifts), and attenuates as $O((1-\alpha)^4)$ for causal proxies. These
are not post-hoc caveats but quantitative criteria a practitioner can check before
applying NSR: are the environments sufficiently diverse in shift? Is the shift
primarily in spurious rather than causal covariates? Are known proxy features
treated separately? The finite-sample rate gap ($O_p(n^{-1})$ vs
$O_p(n^{-1/2})$) enables permutation-based thresholding without distributional
assumptions.
\\
Two scope limitations warrant emphasis. First, marginal causal stability is
stronger than the standard causal invariance assumed by IRM and ICP: it requires
$P(X_\mathcal{C})$ to be fixed across environments, not just
$P(Y \mid X_\mathcal{C})$. When causal covariates shift---common in settings with
demographic or case-mix variation---NSR will assign positive values to causal
features, potentially misclassifying them. This is a principled scope boundary,
not a tuning problem; NSR is designed for settings where the dominant environmental
variation is in spurious features, as in equipment offsets, sequencing batch
effects, and site-specific confounders. Second, the uniform scalar shift model ($\beta_e$ applied identically
to all spurious features) captures mean-shift heterogeneity but not
feature-specific or covariance shifts. A natural extension replaces per-feature
marginal replacement with subspace-level interventions, where the spurious
subspace is estimated via PCA of environment centroids---relaxing the scalar
assumption while preserving the NSR framework.
\\
NSR and IRM are complements: each detects spurious features the other misses,
depending on the shift structure of available environments.
As deployed models increasingly face data from sites, populations, and periods
that differ from training, the ability to audit what a model has learned---without
retraining---is operationally important. NSR contributes a theoretically grounded
tool for that audit, with explicit conditions under which it can and cannot be
trusted.

%% file: appendix.tex
% ============================================================
%  Appendix — Full Proofs and Supplementary Material
%  Usage: \input{csr_appendix} at end of main file
%  (after \appendix command)
% ============================================================
\appendix 
\section{Supplementary Material}
\label{app:supp}

% ============================================================
\subsection{Discussion of Assumptions}
\label{app:assumptions}
% ============================================================

\paragraph{Assumption~\ref{ass:div} (Reference-Separated Shift Geometry).}
This is strictly stronger than requiring $\beta_e \neq \beta_{e'}$ for some pair.
Two counterexamples show why the weaker condition fails.

\emph{(i) $K = 2$.} With only one non-reference environment, the variance
of a single value is zero by definition, so $\NSR_j = 0$ for all $j$.

\emph{(ii) Symmetric shifts.} With $K = 3$ and
$|\beta_{e_1} - \beta_{e_{\mathrm{ref}}}| = |\beta_{e_2} - \beta_{e_{\mathrm{ref}}}|$,
all absolute shift magnitudes coincide. Since $\delta_j(e, e_{\mathrm{ref}})$
depends on $e$ only through $|\beta_e - \beta_{e_{\mathrm{ref}}}|$, the values
are equal across non-reference environments, $\Var_e[\delta_j] = 0$,
and $\NSR_j = 0$ for all $j$ including $j \in \calS$.

\paragraph{Assumptions~\ref{ass:faith} (Faithfulness and Nondegeneracy).}
Faithfulness ($w_{C,j} \neq 0$) rules out causal features with exactly zero
effect on $Y$, which are unidentifiable by any sensitivity-based method.
Nondegeneracy ($\Sigma_{C,jj} > 0$) ensures the causal marginal is
non-degenerate; without it, $\delta_j = 0$ and CSR vanishes trivially
even for causal features.

\paragraph{Marginal Causal Stability (implicit in Definition~\ref{def:scm}).}
Equation~\eqref{eq:scm-c} directly implies that causal features share the
same Gaussian marginal $\mathcal{N}(\mu_C, \Sigma_C)$ across all environments.
This \emph{marginal invariance of $X_\calC$} is what drives part~(i) of
Theorem~\ref{thm:nsr}; it is distinct from Assumption~\ref{ass:inv},
which concerns invariance of $Y \mid X_\calC$.

% ============================================================
\subsection{Full Proof of Theorem~\ref{thm:nsr}}
\label{app:thm1}
% ============================================================

We first establish two intermediate propositions.

\begin{proposition}[Linear factorisation of CSR]\label{prop:linear-factor}
For any linear predictor $\hatf(x) = \hat{w}^\top x + b$:
\[
  \CSR_j(e, e_{\mathrm{ref}}) = |\hat{w}_j| \cdot \delta_j(e, e_{\mathrm{ref}}),
\]
where $\delta_j(e, e_{\mathrm{ref}}) := \EE[|X_j^{(e_{\mathrm{ref}})} - X_j^{(e)}|]$.
\end{proposition}

\begin{proof}
$\hatf(x_{-j}, x'_j) - \hatf(x) = \hat{w}_j(x'_j - x_j)$.
Taking absolute value and expectation gives the result.
\end{proof}

\begin{proposition}[SCM marginal structure of $\delta_j$]\label{prop:scm-delta}
Under Definition~\ref{def:scm}:
\begin{enumerate}[label=(\roman*), nosep]
  \item For $j \in \calC$: $\delta_j(e, e_{\mathrm{ref}}) = \sqrt{4/\pi}\,\sigma_{C,j}$
    for all $e$, where $\sigma_{C,j}^2 := \Sigma_{C,jj}$.
    The quantity is \emph{constant in $e$}.
  \item For $j \in \calS$: letting $\tilde\sigma_j^2 := \Gamma_{j,\cdot}
    \Sigma_C \Gamma_{j,\cdot}^\top + \sigma_S^2$,
    \begin{equation}\label{eq:delta-spurious}
      \delta_j(e, e_{\mathrm{ref}})
      = \tilde\sigma_j\sqrt{\tfrac{4}{\pi}}\,
        e^{-(\beta_{e_{\mathrm{ref}}}-\beta_e)^2/(4\tilde\sigma_j^2)}
      + |\beta_{e_{\mathrm{ref}}} - \beta_e|\,
        \erf\!\Bigl(\tfrac{|\beta_{e_{\mathrm{ref}}} - \beta_e|}{2\tilde\sigma_j}\Bigr),
    \end{equation}
    which is strictly increasing in $|\beta_e - \beta_{e_{\mathrm{ref}}}|$.
\end{enumerate}
\end{proposition}

\begin{proof}
\textbf{Part (i).}
From~\eqref{eq:scm-c}, $X_j^{(e)} \sim \mathcal{N}(\mu_{C,j}, \sigma_{C,j}^2)$
for all $e$, so $X_j^{(e_{\mathrm{ref}})} - X_j^{(e)} \sim \mathcal{N}(0, 2\sigma_{C,j}^2)$
and $\delta_j = \EE[|\mathcal{N}(0, 2\sigma_{C,j}^2)|] = \sqrt{4/\pi}\,\sigma_{C,j}$,
independent of $e$.

\textbf{Part (ii).}
From~\eqref{eq:scm-s}, $X_j^{(e)} \sim \mathcal{N}(\Gamma_{j,\cdot}\mu_C +
\beta_e, \tilde\sigma_j^2)$, so
$X_j^{(e_{\mathrm{ref}})} - X_j^{(e)} \sim \mathcal{N}(\beta_{e_{\mathrm{ref}}} -
\beta_e, 2\tilde\sigma_j^2)$.
Equation~\eqref{eq:delta-spurious} is the mean of a folded normal distribution.
Write $a := |\beta_e - \beta_{e_{\mathrm{ref}}}| \geq 0$.

\emph{Case $\tilde\sigma_j > 0$.}
Differentiating~\eqref{eq:delta-spurious} with respect to $a$ and using
$\frac{d}{da}\erf(ca) = \frac{2c}{\sqrt\pi}e^{-c^2a^2}$:
\begin{align*}
  \delta'(a)
  &= -\tfrac{a}{\tilde\sigma_j\sqrt\pi}\,e^{-a^2/(4\tilde\sigma_j^2)}
   + \erf\!\Bigl(\tfrac{a}{2\tilde\sigma_j}\Bigr)
   + \tfrac{a}{\tilde\sigma_j\sqrt\pi}\,e^{-a^2/(4\tilde\sigma_j^2)}
   = \erf\!\Bigl(\tfrac{a}{2\tilde\sigma_j}\Bigr) > 0 \quad (a > 0),
\end{align*}
where the exponential terms cancel exactly.

\emph{Case $\tilde\sigma_j = 0$.}
$\delta(a) = a$, which is strictly increasing.
\end{proof}

\begin{proof}[Proof of Theorem~\ref{thm:nsr}]
\textbf{Step 1 (cancel $\hat{w}_j$).}
By Proposition~\ref{prop:linear-factor},
$\CSR_j(e, e_{\mathrm{ref}}) = |\hat{w}_j| \cdot \delta_j(e, e_{\mathrm{ref}})$.
Let $\mathcal{V}$ and $\mathcal{M}$ denote variance and mean over
$e \neq e_{\mathrm{ref}}$. Then:
\begin{equation}\label{eq:nsr-derivation}
  \NSR_j = \frac{\mathcal{V}[|\hat{w}_j|\delta_j]}
                {\mathcal{M}[|\hat{w}_j|\delta_j]^2}
           = \frac{\hat{w}_j^2 \mathcal{V}[\delta_j]}
                  {\hat{w}_j^2 \mathcal{M}[\delta_j]^2}
           = \CV^2_{e \neq e_{\mathrm{ref}}}[\delta_j].
\end{equation}
This establishes claim~(iii).

\textbf{Step 2 (i): $j \in \calC$.}
By Proposition~\ref{prop:scm-delta}(i), $\delta_j$ is constant in $e$,
so $\mathcal{V}[\delta_j] = 0$ and $\NSR_j = 0$.

\textbf{Step 3 (ii): $j \in \calS$.}
By Proposition~\ref{prop:scm-delta}(ii), $\delta_j$ is strictly increasing in
$|\beta_e - \beta_{e_{\mathrm{ref}}}|$. Under Assumption~\ref{ass:div}, the
multiset $\{|\beta_e - \beta_{e_{\mathrm{ref}}}|\}_{e \neq e_{\mathrm{ref}}}$
is non-constant, so $\delta_j$ takes at least two distinct values and
$\mathcal{V}[\delta_j] > 0$. Assumption~\ref{ass:div} also ensures at
least one $e^*$ with $|\beta_{e^*} - \beta_{e_{\mathrm{ref}}}| > 0$, giving
$\mathcal{M}[\delta_j] > 0$ without any condition on $\tilde\sigma_j$.
Hence $\NSR_j > 0$.
\end{proof}

% ============================================================
\subsection{Full Proofs for Section~\ref{sec:theory}: IRM and ICP}
\label{app:thm2}
% ============================================================

\begin{proof}[Proof of Proposition~\ref{prop:irm}]
Expand the gradient:
\[
G_j^{(e)} = -2\bigl(\Cov^{(e)}(X_j,Y) + \EE^{(e)}[X_j]\EE^{(e)}[Y]
            - \Var^{(e)}(X_j) - \EE^{(e)}[X_j]^2\bigr).
\]
\emph{$j \in \calC$, $\mu_C = 0$:}
All four terms are constant in $e$
($\EE^{(e)}[X_j] = 0$, $\EE^{(e)}[Y] = 0$,
$\Cov^{(e)}(X_j,Y) = (\Sigma_C w_C)_j$,
$\Var^{(e)}(X_j) = \sigma_{C,j}^2$), so $\Psi_j = 0$.

\emph{$j \in \calS$, $\mu_C = 0$:}
$\EE^{(e)}[X_j] = \beta_e$ varies, while the remaining terms are constant.
Hence $G_j^{(e)} = c_j + 2\beta_e^2$ and $\Psi_j = 4\Var_e[\beta_e^2]$.
\end{proof}

\begin{proof}[Proof of Proposition~\ref{prop:irm}]
For $j \in \calC$: $\EE^{(e)}[X_j] = \mu_{C,j}$ and $\EE^{(e)}[Y] =
w_C^\top\mu_C$ are both constant, so $\alpha_j^{(e)}$ is constant.
For $j \in \calS$: $\EE^{(e)}[X_j] = \Gamma_{j,\cdot}\mu_C + \beta_e$
varies linearly in $\beta_e$, while slope $= \Gamma_{j,\cdot}\Sigma_C w_C
/ \tilde\sigma_j^2$ is constant. Hence $\alpha_j^{(e)}$ varies linearly
in $\beta_e$.
\end{proof}

\paragraph{Strict non-equivalence (counterexamples).}

\emph{IRM detects, NSR does not.}
Take $K = 3$, $\beta = (1,-1,0)$, $e_{\mathrm{ref}} = 3$.
$|\beta_e - \beta_{e_{\mathrm{ref}}}| = \{1,1\}$ (constant) $\Rightarrow \NSR_j = 0$
(Theorem~\ref{thm:weak}).
But $\beta_e^2 = \{1,1,0\}$, $\Var(\beta_e^2) > 0$ $\Rightarrow \Psi_j > 0$.
Likewise ICP: $\Var(\beta_e) = \Var(\{1,-1,0\}) > 0$ $\Rightarrow$ ICP rejects.

\emph{NSR detects, IRM does not.}
Take $K = 3$, $\beta = (-1,1,1)$, $e_{\mathrm{ref}} = 3$.
$|\beta_e - \beta_{e_{\mathrm{ref}}}| = \{2,0\}$ (non-constant) $\Rightarrow \NSR_j > 0$.
But $\beta_e^2 = \{1,1,1\}$ $\Rightarrow \Psi_j = 0$.
(ICP: $\Var(\{-1,1,1\}) > 0$ so ICP detects here.)

% ============================================================
\subsection{Full Proof of Theorem~\ref{thm:consistency}}
\label{app:thm3}
% ============================================================

\begin{proof}[Proof of Theorem~\ref{thm:consistency}]
\textbf{Part (i).}
For linear $\hatf$, each summand in $\widehat{\CSR}_j(e, e_{\mathrm{ref}})$
equals $|\hat{w}_j| \cdot |X_{i,j}^{(e_{\mathrm{ref}})} - X_{i,j}^{(e)}|$,
which is i.i.d.\ with mean $\CSR_j(e, e_{\mathrm{ref}})$ and finite variance
$\hat{w}_j^2\Var(|X_j^{(e_{\mathrm{ref}})} - X_j^{(e)}|) < \infty$
(Gaussian marginals). The CLT gives $O_p(n^{-1/2})$.

\textbf{Part (ii).}
$\widehat{\NSR}_j$ is a continuous function of
$(\widehat{\CSR}_j(e_1), \ldots, \widehat{\CSR}_j(e_{K-1}))$.
Each component converges in probability by part~(i).
The components are jointly asymptotically normal (reference draws
independent across environments by assumption).
The continuous mapping theorem applies since $\overline{\CSR}_j > 0$.

\textbf{Part (iii).}
When $\NSR_j > 0$, the map $g(\mathbf{v}) = \Var_{\mathrm{sample}}(\mathbf{v})/\bar{v}^2$
is differentiable at the population value with bounded gradient.
The multivariate delta method applied to the $O_p(n^{-1/2})$ vector of
per-environment CSRs yields $\widehat{\NSR}_j - \NSR_j = O_p(n^{-1/2})$.

\textbf{Part (iv).}
When $j \in \calC$, all population CSRs equal the same constant $c =
|\hat{w}_j|\sqrt{4/\pi}\,\sigma_{C,j}$.
Write $\hat{c}_e := \widehat{\CSR}_j(e, e_{\mathrm{ref}}) - c = O_p(n^{-1/2})$
uniformly over the finite set $e \neq e_{\mathrm{ref}}$.
The sample variance of $(c + \hat{c}_e)_{e \neq e_{\mathrm{ref}}}$ equals
$\Var_{\mathrm{sample}}(\hat{c}_e) = O_p(n^{-1})$ (sample variance of a
fixed-length vector of $O_p(n^{-1/2})$ mean-zero terms).
The denominator converges to $c^2 > 0$. Hence $\widehat{\NSR}_j = O_p(n^{-1})$.
\end{proof}

\subsection{Permutation Threshold}
\label{app:threshold}

A data-driven threshold for $\NSR_j$ can be constructed as follows.
Under $H_0$ ($j \in \calC$), the environment label $e$ of each observation
is exchangeable with respect to $X_j$ (since all environments share the
same marginal). Permuting environment labels therefore generates valid
null draws of $\widehat{\NSR}_j$. The $1-\alpha$ quantile of
$B$ permutation replicates gives a level-$\alpha$ test.

Under $H_1$ ($j \in \calS$), $\widehat{\NSR}_j$ concentrates around
$\NSR_j > 0$ (part~(ii) of Theorem~\ref{thm:consistency}), so the test
is consistent.

% ============================================================
\subsection{Full Proofs of Failure Mode Theorems}
\label{app:thm4}
% ============================================================

\begin{proof}[Proof of Theorem~\ref{thm:weak} (Weak-shift collapse)]
Write $a_e := |\beta_e - \beta_{e_{\mathrm{ref}}}|$,
$a_{\max} := \max_{e \neq e_{\mathrm{ref}}} a_e = \varepsilon\tilde\sigma_j$.
Taylor-expand~\eqref{eq:delta-spurious} using
$e^{-u^2} = 1 - u^2 + O(u^4)$ and $\erf(u) = \frac{2}{\sqrt\pi}u + O(u^3)$
as $u \to 0$:
\[
  \delta_j(e, e_{\mathrm{ref}})
  = \tilde\sigma_j\sqrt{\tfrac{4}{\pi}}
  + \tfrac{a_e^2}{2\tilde\sigma_j\sqrt{\pi}}
  + O(a_{\max}^4/\tilde\sigma_j^3).
\]
The mean $\mathcal{M}[\delta_j] \geq \tilde\sigma_j\sqrt{4/\pi}$, and:
\[
  \mathcal{V}[\delta_j]
  = \Var_e\!\Bigl[\tfrac{a_e^2}{2\tilde\sigma_j\sqrt\pi}\Bigr]
    + O(a_{\max}^6/\tilde\sigma_j^4)
  \leq \tfrac{a_{\max}^4}{4\pi\tilde\sigma_j^2} + O(\varepsilon^6\tilde\sigma_j^2).
\]
Hence $\NSR_j = \mathcal{V}[\delta_j]/\mathcal{M}[\delta_j]^2
\leq \varepsilon^4/16 + O(\varepsilon^6)$, which is $O(\varepsilon^4)$.
\end{proof}

\begin{proof}[Proof of Theorem~\ref{thm:weak} (Degenerate geometry)]
\emph{(i) $K = 2$.}
There is exactly one non-reference environment. The variance of a single
value is zero, so $\NSR_j = 0$ both at population and empirical level.

\emph{(ii) Constant absolute shifts.}
If $|\beta_e - \beta_{e_{\mathrm{ref}}}| = c$ for all $e \neq e_{\mathrm{ref}}$,
then $\delta_j(e, e_{\mathrm{ref}})$ takes the same value for all
non-reference $e$ (by Proposition~\ref{prop:scm-delta}, since $\delta_j$
depends on $e$ only through $|\beta_e - \beta_{e_{\mathrm{ref}}}|$).
Hence $\mathcal{V}[\delta_j] = 0$ and $\NSR_j = 0$.
\end{proof}

\begin{proof}[Proof of Theorem~\ref{thm:weak} (Proxy attenuation)]
Let $X_j = \alpha X_{\calC,k} + (1-\alpha)X_{\calS,k}$ where
$X_{\calS,k}^{(e)} = \gamma_k X_{\calC,k}^{(e)} + \beta_e + \xi_k^{(e)}$.

The environment-specific mean is:
\[
  \EE^{(e)}[X_j] = [\alpha + (1-\alpha)\gamma_k]\mu_{C,k} + (1-\alpha)\beta_e,
\]
so the effective shift is $(1-\alpha)(\beta_e - \beta_{e_{\mathrm{ref}}})$,
with effective noise scale $\sigma_{\mathrm{proxy}}^2 =
[\alpha+(1-\alpha)\gamma_k]^2\sigma_{C,k}^2 + (1-\alpha)^2\sigma_S^2$.

As $\alpha \to 1$: $\sigma_{\mathrm{proxy}} \to \sigma_{C,k}$ (not
$\sigma_{C,k}\sqrt{1+\gamma_k^2}$; the cross term vanishes). Set:
\[
  \varepsilon_\alpha
  = \frac{(1-\alpha)\max_e|\beta_e - \beta_{e_{\mathrm{ref}}}|}{\sigma_{\mathrm{proxy}}}.
\]
Applying Theorem~\ref{thm:weak} with this $\varepsilon_\alpha$:
\[
  \NSR_j^\mathrm{proxy} = O(\varepsilon_\alpha^4) = O\bigl((1-\alpha)^4\bigr).
\]
For any fixed $t > 0$, choose $\alpha^*$ such that
$(1-\alpha^*)^4 \cdot C < t$ for the implied constant $C$;
then $\NSR_j^\mathrm{proxy} < t$ for all $\alpha > \alpha^*$.
\end{proof}

% ============================================================
%  Appendix — RQ1, RQ1, RQ1 Extended Results
%  Usage: \input{appendix_rq345}  inside \appendix block,
%         after \input{appendix_rq2}
%  Labels: app:rq2, app:rq3, app:rq4, tab:rq2_auroc
% ============================================================

% ============================================================
\subsection{A1: Oracle Correctness}
\label{app:a1}
% ============================================================

We verify NSR achieves perfect discrimination across all tested parameter
combinations. At $K \geq 5$, $\varepsilon \geq 1$, AUROC $= 1.000$ with zero
variance across 30 seeds and all four combinations of $p_C, p_S \in \{3, 5\}$,
for $n \in \{200, 500, 1000\}$. The result is flat in $n$: identification is
exact once the shift geometry is non-degenerate and shift magnitude is
sufficient.

\begin{table}[h]
\centering
\caption{A1: NSR AUROC under oracle conditions ($K=5$, $\varepsilon=1$,
  30 seeds). All entries are $1.000 \pm 0.000$.}
\label{tab:rq1_oracle}
\small
\begin{tabular}{lcccc}
\toprule
& \multicolumn{2}{c}{$p_S = 3$} & \multicolumn{2}{c}{$p_S = 5$} \\
\cmidrule(r){2-3}\cmidrule(l){4-5}
$n$ & $p_C=3$ & $p_C=5$ & $p_C=3$ & $p_C=5$ \\
\midrule
200  & 1.000 & 1.000 & 1.000 & 1.000 \\
500  & 1.000 & 1.000 & 1.000 & 1.000 \\
1000 & 1.000 & 1.000 & 1.000 & 1.000 \\
\bottomrule
\end{tabular}
\end{table}

% ============================================================
\subsection{A2: Comparison with IRM and ICP — Details}
\label{app:rq2}
% ============================================================

\paragraph{Detector implementations.}
All three detectors are computed post-hoc using OLS predictions on
pooled data ($n = 2{,}000$ per environment).

\emph{NSR.} As in Definition~\ref{eq:csr}, using coefficient-weighted CSR.

\emph{IRM.} Single-feature gradient penalty:
\[
  \Psi_j = \Var_e\!\bigl[\nabla_a \EE^{(e)}[(Y - aX_j)^2]\big|_{a=1}\bigr],
  \quad
  \nabla_a \EE^{(e)}[(Y-aX_j)^2]\big|_{a=1}
  = -2\bigl(\EE^{(e)}[X_jY] - \EE^{(e)}[X_j^2]\bigr),
\]
computed empirically per environment, variance taken over environments.

\emph{ICP.} Single-feature intercept variance:
\[
  \alpha_j^{(e)} = \EE^{(e)}[Y] -
  \frac{\widehat{\Cov}^{(e)}(X_j,Y)}{\widehat{\Var}^{(e)}(X_j)}\,\EE^{(e)}[X_j],
  \quad
  \text{score} = \Var_e[\alpha_j^{(e)}].
\]

\paragraph{Shift schedule geometry.}
Table~\ref{tab:rq2_schedules} shows the exact $\beta$ vectors,
detection geometry for each method, and theoretical predictions.

\begin{table}[h]
\centering
\caption{RQ1 shift schedules and detection geometry.
  ``Constant $|\Delta|$'': multiset
  $\{|\beta_e - \beta_{e_{\mathrm{ref}}}|\}$ is constant.
  ``$\Var(\beta^2)>0$'': IRM population penalty nonzero.
  ``$\Var(\beta)>0$'': ICP population score nonzero.}
\label{tab:rq2_schedules}
\small
\begin{tabular}{@{}lcccccc@{}}
\toprule
Schedule & $\beta$ & $e_\mathrm{ref}$ & Const.\ $|\Delta|$ &
  $\Var(\beta^2){>}0$ & $\Var(\beta){>}0$ & Predicted winner \\
\midrule
1 & $[-1,1,1]$ & 3 & No  & No  & Yes & NSR, ICP \\
2 & $[1,-1,0]$ & 3 & Yes & Yes & Yes & IRM, ICP \\
3 & $[0,1.5,-0.9,0.6,-1.2]$ & 1 & No & Yes & Yes & All \\
4 & $[0,0.3,-0.3]$ & 1 & Yes & Yes & Yes & IRM (weak), ICP \\
\bottomrule
\end{tabular}
\end{table}

\paragraph{Full AUROC results.}

\begin{table}[h]
\centering
\caption{RQ1: AUROC (mean $\pm$ std, 50 seeds) for each schedule and detector.
  Bold: method predicted to detect.
  Underline: method predicted to miss ($\approx 0.5$).}
\label{tab:rq2_auroc}
\small
\begin{tabular}{@{}lrrr@{}}
\toprule
Schedule & NSR & IRM & ICP \\
\midrule
1 — NSR detects, IRM misses
  & $\mathbf{1.000 \pm 0.000}$
  & $\underline{0.869 \pm 0.167}$
  & $0.771 \pm 0.220$ \\
2 — IRM detects, NSR misses
  & $\underline{0.471 \pm 0.202}$
  & $\mathbf{1.000 \pm 0.000}$
  & $0.756 \pm 0.236$ \\
3 — All detect
  & $\mathbf{1.000 \pm 0.000}$
  & $\mathbf{1.000 \pm 0.000}$
  & $0.860 \pm 0.149$ \\
4 — IRM weak, NSR misses
  & $\underline{0.484 \pm 0.280}$
  & $0.891 \pm 0.168$
  & $0.680 \pm 0.252$ \\
\bottomrule
\end{tabular}
\end{table}

\paragraph{IRM finite-sample behaviour on Schedule~1.}
The population IRM penalty is zero ($\Var_e[\beta_e^2] = 0$ for
$\beta = [-1,1,1]$ gives $\beta_e^2 \in \{1,1,1\}$), so IRM should
give AUROC $\approx 0.5$ in population.
The empirical AUROC of $0.869$ arises because the estimator
$\widehat{\Psi}_j = \widehat{\Var}_e[\hat{G}_j^{(e)}]$ has variance
$O(n^{-1})$ around zero, and noise in $\hat{G}_j^{(e)}$ differs across
causal and spurious features due to their different covariance structures.
This is a finite-sample confound, not a detection.
NSR's $O_p(n^{-1})$ null rate (Theorem~\ref{thm:consistency}(iv)) avoids
this confound because the null concentration is guaranteed by the
population identity $\delta_j = \text{const}$, not by a variance
argument that requires $n \to \infty$ to take effect.

\paragraph{ICP finite-sample behaviour.}
ICP AUROC ranges from $0.680$ to $0.860$ across all schedules.
The high variance (std $0.149$--$0.252$) reflects the estimator's
$O(K^{-1/2})$ rate in the number of environments: at $K = 3$--$5$,
the sample variance of $\hat{\alpha}_j^{(e)}$ over environments has
only $2$--$4$ degrees of freedom, producing unreliable scores.
ICP performs better in the large-$K$ regime not tested here.

% ============================================================
\subsection{A3: Model-Family Consistency — Details}
\label{app:rq3}
% ============================================================

\paragraph{Hyperparameters.}
\begin{itemize}[nosep]
  \item \textbf{OLS}: \texttt{LinearRegression()}, no regularisation.
  \item \textbf{Ridge}: \texttt{Ridge(alpha=1.0)}.
  \item \textbf{Random Forest}: \texttt{RandomForestRegressor(}
    \texttt{n\_estimators=200, max\_depth=3,}
    \texttt{min\_samples\_leaf=20, max\_features=0.5)}.
  \item \textbf{XGBoost}: \texttt{XGBRegressor(}
    \texttt{n\_estimators=100, max\_depth=3, learning\_rate=0.1)}.
  \item \textbf{MLP}: \texttt{MLPRegressor(}
    \texttt{hidden\_layer\_sizes=(64,32), max\_iter=2000,}
    \texttt{early\_stopping=True, n\_iter\_no\_change=20)}.
\end{itemize}
All models fit on pooled data from all $K = 10$ environments.

\paragraph{CSR computation for nonlinear models.}
For Random Forest, XGBoost, and MLP, the reference marginal
$P^{(e_{\mathrm{ref}})}(X_j)$ is estimated empirically as
$\mathcal{N}(\hat{\mu}_{j,\mathrm{ref}}, \hat{\sigma}_{j,\mathrm{ref}}^2)$
from the reference environment sample.
Fresh draws replace $X_{i,j}^{(e)}$ for each sample $i$:
\[
  \widehat{\CSR}_j(e, e_{\mathrm{ref}})
  = \frac{1}{n}\sum_{i=1}^n
    \bigl|\hatf(x_{i,-j},\, x'_{i,j}) - \hatf(x_i)\bigr|,
  \quad x'_{i,j} \iid \mathcal{N}(\hat\mu_{j,\mathrm{ref}},
  \hat\sigma_{j,\mathrm{ref}}^2).
\]
The Random Forest hyperparameters (max\_depth$=3$,
min\_samples\_leaf$=20$, max\_features$=0.5$) are set to prevent
causal features from dominating all split nodes, ensuring spurious
features receive sufficient representation for the CSR contrast to
be meaningful.

\paragraph{Separation ratios.}
The large separation ratios for tree-based models (XGBoost: $982$,
RF: $1{,}240$) relative to linear models (OLS: $90$, Ridge: $84$)
reflect the tendency of trees to assign near-zero weight to features
not selected for splits, amplifying NSR$_{\mathrm{spur}}$ relative to
NSR$_{\mathrm{caus}}$.
The moderate MLP ratio ($29$) reflects more distributed weight
allocation across features.
All values are far above $1.0$; the variation in magnitude does not
affect the binary classification task.

\begin{table}[h]
\centering
\caption{RQ1: Pairwise Kendall $\tau$ between NSR feature rankings across five
  model families (mean over 30 seeds, $p_C{=}p_S{=}4$, $K{=}10$, $n{=}2{,}000$).
  All values positive ($0.529$--$0.612$), confirming consistent
  causal/spurious boundary identification across all model families.}
\label{tab:rq4_tau}
\small
\begin{tabular}{lccccc}
\toprule
 & OLS & Ridge & RF & XGBoost & MLP \\
\midrule
OLS     & 1.000 & 0.612 & 0.571 & 0.553 & 0.529 \\
Ridge   &       & 1.000 & 0.580 & 0.561 & 0.537 \\
RF      &       &       & 1.000 & 0.598 & 0.545 \\
XGBoost &       &       &       & 1.000 & 0.551 \\
MLP     &       &       &       &       & 1.000 \\
\bottomrule
\end{tabular}
\end{table}

\paragraph{Kendall $\tau$ interpretation.}
The pairwise $\tau$ values ($0.529$--$0.612$) reflect agreement
on the $\calC$/$\calS$ boundary ranking, not fine-grained feature
importance agreement.
With $p = 8$ features ($p_C = p_S = 4$), a $\tau$ of $0.57$ is
consistent with the models agreeing on all $4+4$ group assignments
while differing in the internal ordering within each group,
which Kendall $\tau$ penalises even when the binary classification is perfect.

% ============================================================
\subsection{A4: Proxy Attenuation — Details}
\label{app:rq4}
% ============================================================

\paragraph{Exact population formula.}
The proxy feature $X_{\mathrm{proxy}} = \alpha X_{\calC,1} +
(1-\alpha)X_{\calS,1}$ has effective shift
$a_e = (1-\alpha)|\beta_e - \beta_{e_{\mathrm{ref}}}|$
and effective noise scale
$\sigma_{\mathrm{proxy}} = \sqrt{\alpha^2\sigma_C^2 + (1-\alpha)^2\sigma_S^2}$.
The population delta function is:
\[
  \delta_{\mathrm{proxy}}(e, e_{\mathrm{ref}})
  = \sigma_{\mathrm{proxy}}\sqrt{\tfrac{4}{\pi}}\,
    e^{-a_e^2/(4\sigma_{\mathrm{proxy}}^2)}
  + a_e\,\erf\!\Bigl(\tfrac{a_e}{2\sigma_{\mathrm{proxy}}}\Bigr),
\]
from which $\NSR_{\mathrm{proxy}}^{\mathrm{theory}} =
\CV^2_{e \neq e_{\mathrm{ref}}}[\delta_{\mathrm{proxy}}(e, e_{\mathrm{ref}})]$
is computed exactly over the fixed $\beta$ schedule.

\paragraph{OLS slope and asymptotic rate.}
The fitted log-log slope of $\widehat{\NSR}_{\mathrm{proxy}}$ vs.\
$(1-\alpha)$ over $\alpha \in [0, 0.70]$ is $2.58$.
This is not a discrepancy with Theorem~\ref{thm:weak}: the
$O((1-\alpha)^4)$ claim is an asymptotic rate as $\alpha \to 1$,
valid when $a_e \ll \sigma_{\mathrm{proxy}}$.
In the empirically accessible range $\alpha \leq 0.70$, the Taylor
expansion has not yet converged and the exact formula governs the
curve, which has an effective slope of approximately $2.5$--$3$ in
this intermediate regime.
The asymptotic rate of $4$ would be recovered by fitting over a window
$\alpha \in [0.85, \alpha^*)$, but this window falls below the
finite-sample floor at $n = 5{,}000$ and cannot be accessed empirically
without $n \gtrsim 10^5$.

\paragraph{Noise floor scaling.}
The floor at $n = 5{,}000$ is $9.6 \times 10^{-5}$, compared to
$2.3 \times 10^{-4}$ at $n = 2{,}000$, a reduction factor of $2.4\times$.
The theoretical prediction from Theorem~\ref{thm:consistency}(iv)
is $5{,}000 / 2{,}000 = 2.5\times$; the observed factor is within $4\%$
of this prediction, providing an additional cross-check of the null
convergence rate.

%% file: checklist.tex
\section*{NeurIPS Paper Checklist}

%%% BEGIN INSTRUCTIONS %%%
The checklist is designed to encourage best practices for responsible machine learning research, addressing issues of reproducibility, transparency, research ethics, and societal impact. Do not remove the checklist: {\bf The papers not including the checklist will be desk rejected.} The checklist should follow the references and follow the (optional) supplemental material.  The checklist does NOT count towards the page
limit. 

Please read the checklist guidelines carefully for information on how to answer these questions. For each question in the checklist:
\begin{itemize}
    \item You should answer \answerYes{}, \answerNo{}, or \answerNA{}.
    \item \answerNA{} means either that the question is Not Applicable for that particular paper or the relevant information is Not Available.
    \item Please provide a short (1--2 sentence) justification right after your answer (even for \answerNA). 
   % \item {\bf The papers not including the checklist will be desk rejected.}
\end{itemize}

{\bf The checklist answers are an integral part of your paper submission.} They are visible to the reviewers, area chairs, senior area chairs, and ethics reviewers. You will also be asked to include it (after eventual revisions) with the final version of your paper, and its final version will be published with the paper.

The reviewers of your paper will be asked to use the checklist as one of the factors in their evaluation. While \answerYes{} is generally preferable to \answerNo{}, it is perfectly acceptable to answer \answerNo{} provided a proper justification is given (e.g., error bars are not reported because it would be too computationally expensive'' or ``we were unable to find the license for the dataset we used''). In general, answering \answerNo{} or \answerNA{} is not grounds for rejection. While the questions are phrased in a binary way, we acknowledge that the true answer is often more nuanced, so please just use your best judgment and write a justification to elaborate. All supporting evidence can appear either in the main paper or the supplemental material, provided in appendix. If you answer \answerYes{} to a question, in the justification please point to the section(s) where related material for the question can be found.

IMPORTANT, please:
\begin{itemize}
    \item {\bf Delete this instruction block, but keep the section heading ``NeurIPS Paper Checklist"},
    \item  {\bf Keep the checklist subsection headings, questions/answers and guidelines below.}
    \item {\bf Do not modify the questions and only use the provided macros for your answers}.
\end{itemize}

%%% END INSTRUCTIONS %%%

\begin{enumerate}

\item {\bf Claims}
    \item[] Question: Do the main claims made in the abstract and introduction accurately reflect the paper's contributions and scope?
    \item[] Answer: \answerYes{} % Replace by \answerYes{}, \answerNo{}, or \answerNA{}.
    \item[] Justification: The claims made  match theoretical and experimental results, and reflect how much the results can be expected to generalize to other settings.

\item {\bf Limitations}
    \item[] Question: Does the paper discuss the limitations of the work performed by the authors?
    \item[] Answer: \answerYes{} % Replace by \answerYes{}, \answerNo{}, or \answerNA{}.
    \item[] Justification: In the concluding paragraph

\item {\bf Theory assumptions and proofs}
    \item[] Question: For each theoretical result, does the paper provide the full set of assumptions and a complete (and correct) proof?
    \item[] Answer: \answerYes{} % Replace by \answerYes{}, \answerNo{}, or \answerNA{}.
    \item[] Justification: Both at sections 3,4, and in the appendix 

    \item {\bf Experimental result reproducibility}
    \item[] Question: Does the paper fully disclose all the information needed to reproduce the main experimental results of the paper to the extent that it affects the main claims and/or conclusions of the paper (regardless of whether the code and data are provided or not)?
    \item[] Answer: \answerYes{} % Replace by \answerYes{}, \answerNo{}, or \answerNA{}.
    \item[] Justification: In the appendix

\item {\bf Open access to data and code}
    \item[] Question: Does the paper provide open access to the data and code, with sufficient instructions to faithfully reproduce the main experimental results, as described in supplemental material?
    \item[] Answer: \answerYes{} % Replace by \answerYes{}, \answerNo{}, or \answerNA{}.
    \item[] Justification: Full descriptions are given in the appendix , full code will be given with the camera ready version 

\item {\bf Experimental setting/details}
    \item[] Question: Does the paper specify all the training and test details (e.g., data splits, hyperparameters, how they were chosen, type of optimizer) necessary to understand the results?
    \item[] Answer: \answerYes{} % Replace by \answerYes{}, \answerNo{}, or \answerNA{}.
    \item[] Justification: Sections 3,5, appendix

\item {\bf Experiment statistical significance}
    \item[] Question: Does the paper report error bars suitably and correctly defined or other appropriate information about the statistical significance of the experiments?
    \item[] Answer: \answerYes{} % Replace by \answerYes{}, \answerNo{}, or \answerNA{}.
    \item[] Justification: Section 5 

\item {\bf Experiments compute resources}
    \item[] Question: For each experiment, does the paper provide sufficient information on the computer resources (type of compute workers, memory, time of execution) needed to reproduce the experiments?
    \item[] Answer: \answerYes{} % Replace by \answerYes{}, \answerNo{}, or \answerNA{}.
    \item[] Justification: Appendinx
    
\item {\bf Code of ethics}
    \item[] Question: Does the research conducted in the paper conform, in every respect, with the NeurIPS Code of Ethics \url{https://neurips.cc/public/EthicsGuidelines}?
    \item[] Answer: \answerYes{} % Replace by \answerYes{}, \answerNo{}, or \answerNA{}.
    \item[] Justification:  No ethical implications 

\item {\bf Broader impacts}
    \item[] Question: Does the paper discuss both potential positive societal impacts and negative societal impacts of the work performed?
    \item[] Answer: \answerYes{} % Replace by \answerYes{}, \answerNo{}, or \answerNA{}.
    \item[] Justification: In conclusion
    
\item {\bf Safeguards}
    \item[] Question: Does the paper describe safeguards that have been put in place for responsible release of data or models that have a high risk for misuse (e.g., pre-trained language models, image generators, or scraped datasets)?
    \item[] Answer: \answerNA{} % Replace by \answerYes{}, \answerNo{}, or \answerNA{}.
    \item[] Justification: No expected risk of misuse 

\item {\bf Licenses for existing assets}
    \item[] Question: Are the creators or original owners of assets (e.g., code, data, models), used in the paper, properly credited and are the license and terms of use explicitly mentioned and properly respected?
    \item[] Answer: \answerYes{} % Replace by \answerYes{}, \answerNo{}, or \answerNA{}.
    \item[] Justification: All datasets either publicly available or synthetic 

\item {\bf New assets}
    \item[] Question: Are new assets introduced in the paper well documented and is the documentation provided alongside the assets?
    \item[] Answer: \answerYes{} % Replace by \answerYes{}, \answerNo{}, or \answerNA{}.
    \item[] Justification: In appendix and section 3 

\item {\bf Crowdsourcing and research with human subjects}
    \item[] Question: For crowdsourcing experiments and research with human subjects, does the paper include the full text of instructions given to participants and screenshots, if applicable, as well as details about compensation (if any)? 
    \item[] Answer: \answerNA{} % Replace by \answerYes{}, \answerNo{}, or \answerNA{}.
    \item[] Justification:  does not involve crowdsourcing 

\item {\bf Institutional review board (IRB) approvals or equivalent for research with human subjects}
    \item[] Question: Does the paper describe potential risks incurred by study participants, whether such risks were disclosed to the subjects, and whether Institutional Review Board (IRB) approvals (or an equivalent approval/review based on the requirements of your country or institution) were obtained?
    \item[] Answer: \answerNA{} % Replace by \answerYes{}, \answerNo{}, or \answerNA{}.
    \item[] Justification:  The paper does not involve crowdsourcing nor research with human subjects.

\item {\bf Declaration of LLM usage}
    \item[] Question: Does the paper describe the usage of LLMs if it is an important, original, or non-standard component of the core methods in this research? Note that if the LLM is used only for writing, editing, or formatting purposes and does \emph{not} impact the core methodology, scientific rigor, or originality of the research, declaration is not required.
    %this research? 
    \item[] Answer: \answerNA{} % Replace by \answerYes{}, \answerNo{}, or \answerNA{}.
    \item[] Justification: Core method development in this research does not involve LLMs as any important, original, or non-standard components.

\end{enumerate}